\documentclass[letterpage, 11pt, notitlepage]{article}

\usepackage[margin=1.0in]{geometry}

\usepackage{enumitem}
\usepackage{times}
\usepackage{graphicx} 
\usepackage{subfigure}

\usepackage{algorithm}
\usepackage{algorithmic}

\usepackage{hyperref}
\usepackage{breakurl}

\AtBeginShipout{%
  \ifnum\value{page}>1 %
    \typeout{* Additional boxing of page `\thepage'}%
    \setbox\AtBeginShipoutBox=\hbox{\copy\AtBeginShipoutBox}%
  \fi
}

\newcommand{\norm}[1]{\left\|#1\right\|}

\usepackage{epsfig}
\usepackage{amssymb}
\usepackage{amsmath}
\usepackage{amsthm}
\usepackage{amsfonts}
\usepackage{bbding}
\usepackage{array}
\usepackage{caption,tabularx,booktabs}
\usepackage{arydshln}
\usepackage{xargs}                      
\usepackage[dvipsnames]{xcolor}
\usepackage[colorinlistoftodos,prependcaption,textsize=tiny]{todonotes}
\newcommandx{\unsure}[2][1=]{\todo[inline,linecolor=red,backgroundcolor=red!25,bordercolor=red,#1]{#2}}
\newcommandx{\change}[2][1=]{\todo[linecolor=blue,backgroundcolor=blue!25,bordercolor=blue,#1]{#2}}
\newcommandx{\info}[2][1=]{\todo[linecolor=OliveGreen,inline,backgroundcolor=OliveGreen!25,bordercolor=OliveGreen,#1]{#2}}
\newcommandx{\improvement}[2][1=]{\todo[linecolor=Plum,inline,backgroundcolor=Plum!25,bordercolor=Plum,#1]{#2}}

\usepackage[font=small,skip=5pt]{caption}

\newtheorem{thm}{Theorem}
\newtheorem{lem}{Lemma}

\newtheorem{cor}{Corollary}

\newtheorem{defn}{Definition}

\newtheorem{ass}{Assumption}

\makeatletter
\newcounter{subthm}
\let\savedc@thm\c@hyp

\newcommand{\normhyp}{%
  \let\c@hyp\savedc@hyp 
  \renewcommand\thehyp{\arabic{hyp}}%
} 
\makeatother

\makeatletter
\newcounter{subass}
\let\savedc@ass\c@hyp

\makeatother

\newcommand\tagthis{\addtocounter{equation}{1}\tag{\theequation}}

\DeclareMathOperator{\Ocal}{\mathcal{O}}

\newcommand{\m}[1]{{\bf{#1}}}
\newcommand{\tr}{^{\sf T}}
\newcommand{\C}[1]{{\cal {#1}}}

\usepackage{xcolor}

\usepackage{array}
\usepackage{caption,tabularx,booktabs}
\usepackage{arydshln}

\usepackage{xargs}                      
\usepackage[colorinlistoftodos,prependcaption,textsize=tiny]{todonotes}

\makeatother

\usepackage{blindtext,titlefoot}

\begin{document}

\title{\fontsize{20}{20}\selectfont When Does Stochastic Gradient Algorithm Work Well?}

\author{
Lam M. Nguyen
\and
Nam H. Nguyen
\and
Dzung T. Phan
\and
Jayant R. Kalagnanam
\and
Katya Scheinberg
}

\maketitle

\unmarkedfntext{\textbf{Lam M. Nguyen}, Industrial and Systems Engineering, Lehigh University, PA, USA. Email: \href{mailto:lamnguyen.mltd@gmail.com}{LamNguyen.MLTD@gmail.com}}
\unmarkedfntext{\textbf{Nam H. Nguyen}, IBM Thomas J. Watson Research Center, Yorktown Heights, NY, USA. Email: \href{mailto:nnguyen@us.ibm.com}{nnguyen@us.ibm.com}}
\unmarkedfntext{\textbf{Dzung T. Phan}, IBM Thomas J. Watson Research Center, Yorktown Heights, NY, USA. Email: \href{mailto:phandu@us.ibm.com}{phandu@us.ibm.com}}
\unmarkedfntext{\textbf{Jayant R. Kalagnanam}, IBM Thomas J. Watson Research Center, Yorktown Heights, NY, USA. Email: \href{mailto:jayant@us.ibm.com}{jayant@us.ibm.com}}
\unmarkedfntext{\textbf{Katya Scheinberg}, Industrial and Systems Engineering, Lehigh University, PA, USA. Email: \href{mailto:katyas@lehigh.edu}{katyas@lehigh.edu}. The work of this author is partially supported by NSF Grants CCF 16-18717 and 
CCF 17-40796}

\begin{abstract}
In this paper, we consider a general stochastic optimization problem which is often at the core of
supervised learning, such as deep learning and linear classification.  We consider a standard stochastic
gradient descent (SGD) method with a fixed, large step size and propose a novel assumption on the objective function,
under which this method has the improved convergence rates (to a neighborhood of the optimal solutions). We then empirically demonstrate that these assumptions hold for logistic regression and standard  deep neural networks on classical data sets. Thus our analysis helps to explain when efficient behavior can be expected from the SGD method in training classification models and deep neural networks.
\end{abstract}

\section{Introduction and Motivation}\label{intro}

In this paper we are interested in analyzing behavior of the stochastic gradient algorithm when solving empirical and expected risk minimization problems.
For the sake of generality we  consider the following stochastic optimization problem
\begin{align*}
\min_{\m{w} \in \mathbb{R}^d} \left\{ F(\m{w}) = \mathbb{E} [ f(\m{w};\xi) ] \right\}, \tagthis \label{main_prob_expected_risk}
\end{align*}
where $\xi$ is a random variable  obeying some distribution.

 In the case of empirical risk minimization with a training set $\{(\m{x}_i,\m{y}_i)\}_{i=1}^n$, $\xi_i$ is a random variable that is defined
 by a single random sample  $(\m{x},\m{y})$  drawn uniformly from the training set.
Then,  by defining  $f_i(\m{w}) := f(\m{w};\xi_{i})$ we write the empirical risk minimization as follows:
\begin{align*}
\min_{\m{w} \in \mathbb{R}^d} \left\{ F(\m{w}) = \frac{1}{n}
\sum_{i=1}^n f_i(\m{w}) \right\}. \tagthis \label{main_prob_empirical_risk}
\end{align*}

More generally $\xi$ can be a random variable defined by a random subset of samples $\{(\m{x}_i,\m{y}_i)\}_{i \in I}$ drawn from the training set, in which case formulation  \eqref{main_prob_expected_risk} still applies to the empirical risk minimization. On the other hand, if $\xi$ represents a sample or a set of samples
drawn from the data distribution, then \eqref{main_prob_expected_risk} represents the expected risk minimization.


Stochastic gradient descent (SGD), originally introduced in
\cite{RM1951}, has become the method of choice for solving not only \eqref{main_prob_expected_risk} but also
\eqref{main_prob_empirical_risk} when $n$ is large. Theoretical justification for
using  SGD for machine learning problems is given, for example, in
\cite{BousquetBottou}, where it is shown that, at least for convex
problem, SGD is an optimal method for minimizing expected risk,
which is the ultimate goal of learning. From the practical
perspective SGD is often preferred to the standard gradient descent
(GD) method simply because GD requires computation of a full
gradient on each iteration, which, for example, in the case of deep neural networks (DNN),
requires applying backpropagation for all $n$ samples, which can be
prohibitive.

Consequently, due to its simplicity in implementation and efficiency in dealing
with large scale datasets, SGD has become by far the most common method for training deep neural networks and other large scale ML models.
However, it is well  known that SGD can be  slow and unreliable in some practical
applications as its behavior is strongly dependent on the chosen stepsize and on the variance of the stochastic gradients.
While the method may provide fast initial improvement, it may
slow down drastically after a few epochs and can even fail to move close enough to a
solution for a fixed learning rate. To overcome this oscillatory
behavior, several variants of SGD have been recently proposed. For example, methods such
as AdaGrad \cite{AdaGrad}, RMSProp \cite{tielemanH12}, and Adam
\cite{KingmaB14}  adaptively select the
stepsize for each component of $\m{w}$.
Other techniques include  diminishing stepsize scheme
\cite{bottou2016optimization} and variance reduction methods
\cite{SAGjournal, SAGA, SVRG,nguyen2017sarah}.  These latter methods  reduce the
 variance of the stochastic gradient
estimates, by either computing  a full gradient after a certain number of
iterations or by storing the past gradients, both of which can be expensive. Moreover, these methods only apply to the finite sum problem
(\ref{main_prob_empirical_risk})  but not the general problem (\ref{main_prob_expected_risk}).
On the other hand these methods
enjoy faster convergence rates than that of SGD. For example, when
$F(\m{w})$ is strongly convex, convergence rates of the variance
reduction methods (as well as that of GD itself) are linear, while for
SGD it is only sublinear. While GD has to compute the entire
gradient on {\em every} iteration, which makes it more
expensive than the variance reduction methods,  its convergence
analysis allows for a much larger fixed stepsizes than those allowed in
the variance reduction methods.
In this paper we are particularly interested in addressing an observation: a simple SGD with a fixed,
reasonably large, step size can  have a fast convergence rate to some neighborhood of the optimal solutions,
 without resorting to additional procedures for variance
reduction.

Let us consider an example of recovering a signal $\hat{\m{w}} \in \mathbb{R}^2$ from $n$ noisy observations $y_i =  y_i^{\text{clean}} + e_i$ where $y_i^{\text{clean}} = (\m{a}_i\tr \hat{\m{w}})^2$. Here, $\m{a}_i$'s are random vectors and $e_i$'s are noise components. To recover $\hat{\m{w}}$ from the observation vector $\m{y}$, we solve a non-convex fourth-order polynomial minimization problem
$$
\min_{\m{w}} \left\{ F(\m{w}) =
\frac{1}{n} \sum_{i=1}^n (y_i - (\m{a}_i\tr\m{w})^2)^2 \right\}.
$$
Note that there are at least two global solutions to this problem, which we denote $w_*$ and $-w_*$.
We consider two possible scenarios:

\begin{itemize}
\item[(i)] \textit{All} of the component functions $f_i(\m{w})=(y_i - (\m{a}_i\tr\m{w})^2)^2$ have \textit{relatively small} gradients
at both of the optimal solutions $\m{w}_*$ and $-\m{w}_*$ of the aggregate $F(\m{w})$. In this case this means that $w_*$ recovers a good fit for the observations $y$.
\item[(ii)] There are \textit{many} indices $i$ such that at the optimal solutions of $F(\m{w})$, the associated gradients $\nabla f_i$ are \textit{ large}. This happens when $w_*$ does not provide a good fit, which can happen when the noise $e_i$ is large.
\end{itemize}

\begin{figure}[h]
 \begin{center}
 \includegraphics[width=0.49\textwidth]{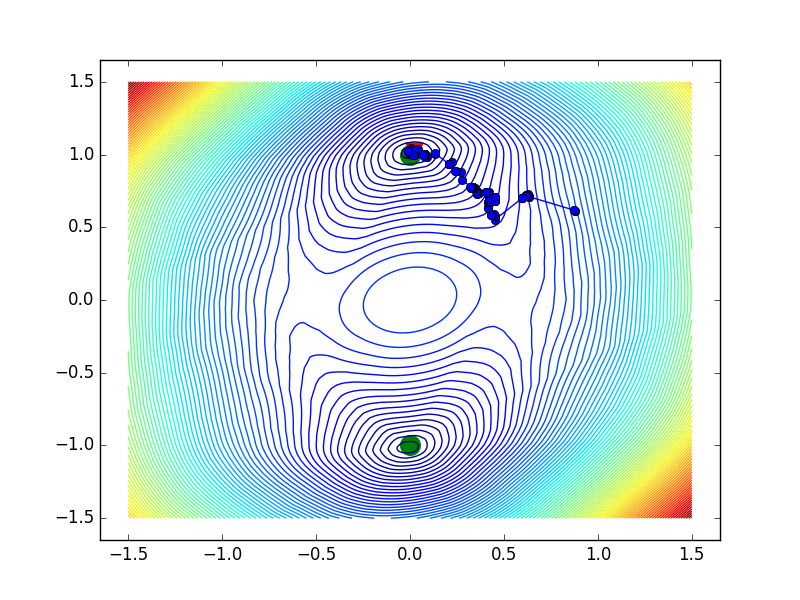}
 \includegraphics[width=0.49\textwidth]{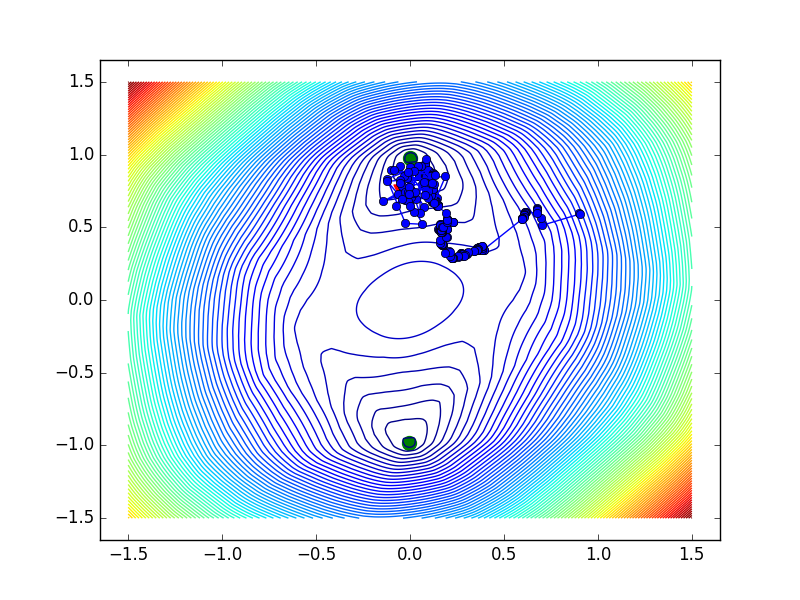}
  \caption{Stochastic Gradient Descent}
  \label{fig:SGD_sample}
  \end{center}
 \end{figure}

We set $n=100$ and generate these two scenarios by setting all the noise components $e_i$ to be  small ($1 \%$ of the energy of $\m{y}^{\text{clean}}$) for case (i) or setting only first $40$ noise components to be large ($25 \%$ of the energy of $\m{y}^{\text{clean}}$) for case (ii). We can observe from Figure \ref{fig:SGD_sample} that SGD algorithm
 converges to the optimal solution of $F(\m{w})$ in  case (i) depicted in
the left figure; but fails to converge  to the solution of $F$ in case
(ii) as shown in the right figure.
The intuition behind this  behavior is as follows. At every step of SGD, the
 iterate  moves towards to the optimal solutions of the
individual component function that has been randomly chosen on this
iteration. If  a majority of component functions $f_i$ have their
optimal solutions close to the optimum of the entire problem $F$,
then SGD effectively acts as GD. On the other hand,  if the optimal
solutions of a lot of  $f_i's $ are far from each other and from the
overall optimum, then iterates of SGD wander randomly in the region
around these individual optima, as shown on the right of Figure
\ref{fig:SGD_sample}. Hence, SGD cannot work effectively in case
(ii), unless we either reduce the learning rate or reduce the
variance of the steps thus attaining more accurate gradient
information.


In this paper we generalize this result for stochastic problem \eqref{main_prob_expected_risk} under much weaker assumptions. In particular, we do not assume that the gradients {\em vanish} at the solution,
but that they are bounded by some small constant. Moreover, we do not impose this property on {\em all} stochastic gradients, but assume that it holds with suitably large probability. We then show that SGD has fast convergence rates in the strongly convex, convex and nonconvex cases, until some accuracy is reached, where this accuracy is dictated by the behavior of the stochastic gradients at the optimal solution.

We conjecture that  success of SGD for training many machine learning models
 is the result of the associated optimization
problems having this properties - most of the component gradients are suitably small at the solution. To verify this claim, we trained
 linear classifiers (via logistic regression) and  standard  neural
networks on several well-known  datasets and subsequently computed
the fraction of individual gradients  $\nabla f_i(\m{w}_*)$
 at the final solution $\m{w}_*$ of $F$, that were small. The results show that more than 99\% of component functions $f_i$
have the vanishing gradient at $\m{w}_*$. More numerical evidence is presented in the Section \ref{sec_experiment}.

Hence we base our analysis on the following observation.

\textbf{Main observation}. \textit{For many classification problems in supervised learning, majority of component functions $f_i$ have
small gradients at the optimal solution $w_*$ (in the convex case) or at  local minima of $F(\m{w})$ (in the nonconvex case)}

In this paper, based on this observation, we  provide theoretical
analysis of SGD under the assumption on the fraction of components with small gradient at the solution.
Our analysis helps explain the good performance of SGD when applied to deep
learning. We summarize the key contributions of the paper as
follows.



\begin{itemize}
\item We conjecture that in many instances of  empirical risk minimization and expected risk minimization
SGD converges  to a neighborhood of a stationary point of $F(\m{w})$ such that
the majority of component functions $f_i$ have small gradients at that point. We
 verify numerically that this conjecture holds true for logistic
regression and standard deep neural networks on a wide range of data
sets.
\item We formalize this conjecture as a condition under which we  are able to
establish improved convergence rates of SGD with fixed, large step size to a neighborhood of such stationary point when $F(\m{w})$ is strongly convex, convex and nonconvex.

\item Thus we establish that SGD converges fast  to a neighborhood of the expected/empirical risk minimizer and that the size of the neighborhood
is determined by some properties of the distribution of the stochastic gradient at the minimizer.
\end{itemize}

The remainder of the paper is organized as follows. The main convergence analysis for all three cases is carried out in Section \ref{sec_conj_evidence}.
The computational evidence is presented in Section \ref{sec_experiment} and implications of our analysis and findings are summarized in Section \ref{sec_conclusion}. The proofs are presented in the Appendix.

\section{Convergence Analyses of Stochastic Gradient Algorithms}
\label{sec_conj_evidence}


In this section, we analyze  the stochastic gradient descent algorithm (SGD) under a novel condition, based on the observations of the previous section, and derive improved convergence rates
  for the strongly convex, convex, and non-convex cases. We present each result
  in the form of a general theorem with the bound on a certain optimality measure (depending on the case), followed by the corollary
   where we  demonstrate that improved convergence rate can be observed until this optimality measure becomes small. The rate and the threshold for optimality measure are dictated by the properties of the stochastic gradient at the solution.

  First we introduce the basic definition of
$L$-smoothness.

\begin{defn}
\label{def:L-smooth}
A function $\phi$ is $L$-smooth if there exists a constant $L > 0$ such
that
\begin{align*}
\| \nabla \phi(\m{w}) - \nabla \phi(\m{w}') \| \leq L \| \m{w} - \m{w}'
\|, \ \forall\; \m{w},\m{w}' \in \mathbb{R}^d.
\tagthis\label{eq:Lsmooth_basic}
\end{align*}
\end{defn}

For completeness, we state the SGD algorithm as Algorithm \ref{sgd_algorithm}.

\begin{algorithm}[h]
   \caption{Stochastic Gradient Descent (SGD) Algorithm with fixed step size }
   \label{sgd_algorithm}
\begin{algorithmic}
   \STATE {\bfseries Initialize} $\m{w}_0$, {\bfseries choose} stepsize $\eta>0$, and batch size $b$.
   \FOR{$t=0,1,2,\dots$}
  \STATE Generate random variables $\left\{\xi_{t,i}\right\}_{i=1}^b$ i.i.d. with $\mathbb{E}[\nabla f(\m{w}_t;\xi_{t,i}) | \mathcal{F}_t] = \nabla F(\m{w}_t)$.
  \STATE Compute a stochastic gradient $$\m{g}_t= \frac{1}{b} \sum_{i=1}^b \nabla f(\m{w}_{t};\xi_{t,i}).$$
   \STATE Update the new iterate $\m{w}_{t+1} = \m{w}_{t} - \eta \m{g}_t$.
   \ENDFOR
\end{algorithmic}
\end{algorithm}



Let $\mathcal{F}_t =
\sigma(\m{w}_0,\m{w}_1,\dots,\m{w}_t)$ be the
$\sigma$-algebra generated by $\m{w}_0,\m{w}_1,\dots,\m{w}_t$. We note that $\left\{\xi_{t,i}\right\}_{i=1}^b$ are independent of $\mathcal{F}_t$. Since $\left\{\xi_{t,i}\right\}_{i=1}^b$ are i.i.d.\footnote{Independent and identically distributed random variables. We note from probability theory that if $X_1,\dots,X_d$ are i.i.d. random variables then $g(X_1),\dots,g(X_d)$ are also i.i.d. random variables if $g$ is measurable function.} with $\mathbb{E}[\nabla f(\m{w}_t;\xi_{t,i}) | \mathcal{F}_t] = \nabla F(\m{w}_t)$, we have an unbiased estimate of gradient $\mathbb{E}[ \m{g}_t | \mathcal{F}_t ] = \frac{1}{b} \sum_{i=1}^b \nabla F(\m{w}_t)   = \nabla F(\m{w}_t)$.

We now define the quantities  that will be useful in our  results.


\begin{defn}
\label{defn_p} Let $\m{w}_*$ be a stationary point of the objective
function $F(\m{w})$. For any given threshold $\epsilon > 0$, define
\begin{align*}
p_{\epsilon} & := \mathbb{P}\left\{ \norm{ \m{g}_* }^2
\leq \epsilon  \right\}, \tagthis \label{eq_Prob_p}
\end{align*}
where $\m{g}_* = \frac{1}{b} \sum_{i=1}^b \nabla f (\m{w}_*; \xi_i)$, as the probability that event $\norm{ \m{g}_* }^2
\leq \epsilon$ happens for some i.i.d. random variables $\{\xi_i\}_{i=1}^b$. We also define
\begin{align*}
M_{\epsilon} := \mathbb{E} \left[ \| \m{g}_* \|^2 \ | \ \| \m{g}_* \|^2 > \epsilon \right]. \tagthis \label{eq_upper_bound}
\end{align*}
\end{defn}

The quantity $p_{\epsilon}$ measures the probability that event $\norm{ \m{g}_* }^2 \leq \epsilon$ happens for some realizations of random variables $\xi_i$, $i=1,\ldots, b$. Clearly, $p_{\epsilon}$ is bounded above by
$1$ and monotonically increasing with respect to $\epsilon$.
Quantity $M_{\epsilon}$ can be interpreted as the average
bound of large components $\| \nabla f (\m{w}_*; \xi) \|^2$.
As  we will see in our results below,
 quantities $p_{\epsilon}$ and
$M_{\epsilon}$ appear in the convergence rate bound of the
SGD algorithm. $M_\epsilon$ is also bounded above by $M_{max}=\max_{\xi} \|\nabla f (\m{w}_*; \xi) \|^2$, which we assume is finite,
hence in all our results we can replace $M_\epsilon$ by $M_{max}$ if we want to eliminate its dependence on $\epsilon$.
 On the other hand, the dependence of quantity $p_{\epsilon}$ on $\epsilon$ is key for our analysis.
 Based on the evidence shown in Section \ref{sec_experiment}, we expect
$p_{\epsilon}$ to be close to $1$ for all but very small values of
$\epsilon$. We will derive our convergence rate bounds in terms of $\max\{\epsilon,1-p_{\epsilon}\}$. Clearly, as $\epsilon$ decreases,  $1-p_{\epsilon}$ increases and vice versa, but if there exists a small $\epsilon$ for which $1-p_{\epsilon}\approx {\epsilon}$ then our  results show convergence of SGD to an $\Ocal(\epsilon)$ neighborhood of the solution, at an improved rate with respect to ${\epsilon}$.

\subsection{Convex objectives}
\label{sec_stronglyconvex}

In this section, we analyze the SGD method in the context of
minimizing a convex objective function. We will bound the expected
optimality gap at a given iterate in terms of the value of
$p_{\epsilon}$.
First, we
consider the case when $F$ is strongly convex.

\begin{defn}
A function $\phi$ is $\mu$-strongly convex if there exists a constant
$\mu
> 0$ such that
\begin{align*}
\phi(\m{w}) - \phi(\m{w}') \geq \nabla \phi(\m{w}')^\top (\m{w} - \m{w}') +
\frac{\mu}{2}\|\m{w} - \m{w}'\|^2, \ \forall \m{w}, \m{w}' \in
\mathbb{R}^d. \tagthis\label{eq:stronglyconvex_00}
\end{align*}
\end{defn}


Using this definition, we state the following result for the
strongly convex case.

\begin{thm}\label{thm_sgd_str_convex_03}
Suppose that $F(\m{w})$ is $\mu$-strongly convex and $f(\m{w};\xi)$ is $L$-smooth and convex for every realization of $\xi$. Consider
Algorithm~\ref{sgd_algorithm} with $\eta \leq \frac{1}{L}$. Then, for any $\epsilon>0$
\begin{align*}
\mathbb{E} [ \|\m{w}_{t} - \m{w}_* \|^2] \leq (1- \mu\eta(1-\eta L) )^t
 \| \m{w}_{0} - \m{w}_* \|^2  + \frac{2\eta }{ \mu(1-\eta L)} p_{\epsilon} \epsilon +
\frac{2\eta}{ \mu(1-\eta L)} (1-p_{\epsilon})M_{\epsilon}, \tagthis \label{eq_str_convex_main}
\end{align*}
where $\m{w}_* = \arg \min_{\m{w}} F(\m{w})$, and $p_{\epsilon}$ and $M_{\epsilon}$ are defined in \eqref{eq_Prob_p} and \eqref{eq_upper_bound}, respectively.

\end{thm}



The main conclusion is stated in the following corollary.

\begin{cor} \label{cor_sgd_str_convex_03_II} For any $\epsilon$ such that $1 - p_{\epsilon} \leq  \epsilon$, and for Algorithm~\ref{sgd_algorithm} with $\eta \leq \frac{1}{2L}$, we have
\begin{align*}
\mathbb{E} [ \|\m{w}_t - \m{w}_* \|^2 ] & \leq (1- \mu\eta)^t \|
\m{w}_{0} - \m{w}_* \|^2 + \frac{2\eta}{\mu}\left (1+ M_{\epsilon}\right )\epsilon.
\end{align*}
Furthermore if $t\geq T$ for $T =  \frac{1}{\mu \eta} \log \left( \frac{\mu\| \m{w}_{0} - \m{w}_* \|^2}{2\eta\left (1+ M_{\epsilon}\right )\epsilon} \right)$, then
 \begin{equation}\label{eq:strongconvbound_II}
 \mathbb{E} [
\|\m{w}_{t} - \m{w}_* \|^2 ] \leq \frac{4\eta}{ \mu}\left (1+M_{\epsilon}\right )  \epsilon .
\end{equation}
\end{cor}

Note that in  Corollary \ref{cor_sgd_str_convex_03_II} we assume that $\eta \leq \frac{1}{2L}$ instead of $\eta \leq \frac{1}{L}$ only to simplify the expressions. (The proof in detail is in the Appendix.) We conclude that under the assumption  $1 - p_{\epsilon} \leq  \epsilon$, Algorithm~\ref{sgd_algorithm}  has linear convergence rate in terms of any such $\epsilon$.

The following theorem establishes convergence rate bound for Algorithm~\ref{sgd_algorithm}  when the strong
convexity assumption on $F(\m{w})$ is relaxed.
\begin{thm}\label{thm_sgd_convex_01}
Suppose that $f(\m{w};\xi)$ is $L$-smooth and convex for every realization of $\xi$.
Consider Algorithm~\ref{sgd_algorithm} with $\eta < \frac{1}{L}$. Then for any $\epsilon>0$, we have
\begin{align*}
\frac{1}{t+1} \sum_{k=0}^t \mathbb{E} [F(\m{w}_k) - F(\m{w}_*)] &
\leq \frac{\| \m{w}_0 - \m{w}_* \|^2}{2\eta(1 - \eta L) t}  +
\frac{\eta}{(1 - \eta L)} p_{\epsilon} \epsilon + \frac{\eta M_{\epsilon}}{(1 - \eta L)} (1 - p_{\epsilon}), \tagthis \label{eq_convex_main}
\end{align*}
where $\m{w}_* $ is any optimal solution of $F(\m{w})$, and  $p_{\epsilon}$ and $M_{\epsilon}$ are defined in \eqref{eq_Prob_p} and \eqref{eq_upper_bound}, respectively.
\end{thm}

Again, the convergence rate of SGD is governed by the initial solution and quantities $p_{\epsilon}$ and $M_{\epsilon}$. Hence we have the following corollary.


\begin{cor}\label{cor_sgd_convex_01}
If $f(\m{w};\xi)$ is $L$-smooth and convex for every realization of $\xi$, then for any $\epsilon$ such  that $1 - p_{\epsilon} \leq  \epsilon$,  and $\eta \leq
\frac{1}{2L}$, it holds that
\begin{align*}
\frac{1}{t+1} \sum_{k=0}^t \mathbb{E} [F(\m{w}_k) - F(\m{w}_*)] &
\leq \frac{\| \m{w}_0 - \m{w}_* \|^2}{\eta t}  +
2\eta \left(1+ M_{\epsilon}  \right) \epsilon.
\end{align*}
Hence, if $t\geq T$ for  $T = \frac{ \| \m{w}_0 - \m{w}_*
\|^2}{(2\eta^2) (1+M_{\epsilon})\epsilon}$, we have
\begin{align*}
\frac{1}{t+1} \sum_{k=0}^t \mathbb{E} [F(\m{w}_k) - F(\m{w}_*)] \leq
4\eta \left(1+M_{\epsilon}\right) \epsilon.
\tagthis \label{eq:convexbound}
\end{align*}
\end{cor}


Similarly to the
strongly convex case, under the key assumption that $1 - p_{\epsilon} \leq  \epsilon$,  we show that Algorithm~\ref{sgd_algorithm} achieves
$\Ocal(\epsilon)$ optimality gap, in expectation, in $ \Ocal(1/\epsilon)$ iterations. In  Corollary \ref{cor_sgd_convex_01} we again assume that $\eta \leq \frac{1}{2L}$ instead of $\eta < \frac{1}{L}$ only to simplify the expressions and to replace $\frac{1}{1-\eta L}$ term with $2$ in the complexity bound.


\subsection{Nonconvex objectives}
\label{sec_nonconvex}

In this section, we establish expected complexity bound for
Algorithm~\ref{sgd_algorithm} when applied to nonconvex objective
functions. This setting includes deep neural networks in which the
cost function is a sum of nonconvex function components. Despite
the nonconvexity of the objective, it has been observed that deep neural networks  can be trained fairly quickly by SGD algorithms. It has also been observed that after reaching certain accuracy, the SGD algorithm may slow down dramatically.


For the analysis of the nonconvex case, we need to make an assumption on the rate of change in the gradients near all local solutions, or at least
those to which  iterates $\m{w}_t$ generated by the algorithm may converge.



%
%

\begin{ass}\label{ass_nonconvex_00}
We assume that there exists a constant $N > 0$, such that for any sequence  of iterates $\m{w}_0$, $\m{w}_1$, $\dots$, $\m{w}_t$  of any realization of Algorithm \ref{sgd_algorithm}, there exists a stationary point $\m{w}_*$ of $F(\m{w})$ (possibly dependent on that sequence) such that
\begin{align*}
\frac{1}{t+1} \sum_{k=0}^t \left( \mathbb{E} \left[ \left \|\frac{1}{b}\sum_{i=1}^b \nabla f(\m{w}_k; \xi_{k,i}) - \frac{1}{b}\sum_{i=1}^b \nabla f(\m{w}_*; \xi_{k,i}) \right \|^2 \Big | \mathcal{F}_k \right] \right) \leq N \frac{1}{t+1} \sum_{k=0}^t \| \nabla F(\m{w}_k) \|^2, \tagthis \label{eq_ass_nonconvex_01}
\end{align*}
where the expectation is taken over random variables $\xi_{k,i}$ conditioned on $\mathcal{F}_t =
\sigma(\m{w}_0,\m{w}_1,\dots,\m{w}_t)$, which is the
$\sigma$-algebra generated by $\m{w}_0,\m{w}_1,\dots,\m{w}_t$.
Let ${\cal W_*}$ denote the set of all such stationary points $\m{w}_*$, determined by the constant $N$ and by  realizations $\m{w}_0$, $\m{w}_1$, $\dots$, $\m{w}_t$.
\end{ass}


This assumption is made for any realization $\m{w}_0$, $\m{w}_1$, $\dots$, $\m{w}_t$ and states that the average squared norm of the difference between the stochastic gradient directions computed by  Algorithm \ref{sgd_algorithm} at $\m{w}_t$ and the same stochastic gradient computed at  $\m{w}_*$,
 over any $t$ iterations, is proportional to the average true squared gradient norm.
 If $\m{w}_*$ is a stationary point for all $f(\m{w}; \xi_{k,i}$), in other words, $\nabla f(\m{w}_*; \xi_{k,i})=0$ for all realizations of $\xi_{k,i}$, then
 Assumption \ref{ass_nonconvex_00} simply states that all stochastic gradients have the same average expected rate of growth as the true gradient, as the iterates get further away from $\m{w}_*$. Notice that $\m{w}_*$ may not be a stationary point for all $f(\m{w}; \xi_{k,i})$, hence Assumption \ref{ass_nonconvex_00} bounds the average expected rate of {\em change} of the stochastic gradients in terms of the rate of change of the true gradient.
 In the next section we will demonstrate numerically that Assumption \ref{ass_nonconvex_00} holds for problems of training deep neural networks.

 We also need to slightly modify Definition \ref{defn_p}.
 \begin{defn}
\label{defn_p2} Let $\m{g}_* = \frac{1}{b} \sum_{i=1}^b \nabla f (\m{w}_*; \xi_i)$ for some i.i.d. random variables $\{\xi_i\}_{i=1}^b$. For any given threshold $\epsilon > 0$, define
\begin{align*}
p_{\epsilon} := \inf_{\m{w_*}\in{\cal W_*}} \mathbb{P}\left\{ \norm{ \m{g}_* }^2
\leq \epsilon  \right\} , \tagthis \label{eq_Prob_p2}
\end{align*}
where the infimum is taken over the set ${\cal W_*}$ defined in Assumption \ref{ass_nonconvex_00}.
 Similarly, we also define
\begin{align*}
M_{\epsilon} :=\sup_{\m{w_*}\in {\cal W_*}} \mathbb{E} \left[ \| \m{g}_* \|^2 \ | \ \| \m{g}_* \|^2 > \epsilon \right]. \tagthis \label{eq_upper_bound2}
\end{align*}
\end{defn}
We know that $p_{\epsilon}$ and $M_{\epsilon}$ defined as above exist since $p_{\epsilon}\geq 0$ and $M_{\epsilon}\leq M_{max}$. This time, if we assume that $1 - p_\epsilon\leq \epsilon$ for some reasonably small $\epsilon$, this implies that for {\em all} stationary points of $F(\m{w})$ that appear
in Assumption \ref{ass_nonconvex_00}  a large fraction of stochastic gradients have small norm at those points. Essentially, ${\cal W_*}$ consists of  stationary points to which different realization of SGD iterates converge.

\begin{thm}\label{thm_nonconvex_05}
Let Assumption \ref{ass_nonconvex_00} hold for some $N>0$. Suppose that $F$ is $L$-smooth. Consider Algorithm \ref{sgd_algorithm} with $\eta < \frac{1}{LN}$. Then, for any $\epsilon > 0$, we have
\begin{align*}
\frac{1}{t+1} \sum_{k=0}^t \mathbb{E}[ \| \nabla F(\m{w}_k) \|^2 ]
&\leq \frac{[F(\m{w}_0) - F^*]}{\eta\left(1 - L\eta N \right) (t + 1)} + \frac{L\eta}{ \left(1 - L\eta N \right)} \epsilon + \frac{L\eta M_{\epsilon}}{\left(1 - L\eta N \right)} ( 1 - p_{\epsilon}),
\end{align*}
where $F^*$ is any lower bound of $F$; and $p_{\epsilon}$  and $M_{\epsilon}$ are defined in \eqref{eq_Prob_p2} and \eqref{eq_upper_bound2} respectively.
\end{thm}

%
\begin{cor}\label{cor_nonconvex_05}
Let Assumption \ref{ass_nonconvex_00} hold and $p_{\epsilon}$ and $M_{\epsilon}$ be defined as in  \eqref{eq_Prob_p2} and \eqref{eq_upper_bound2}.  For any $\epsilon$ such that $1 - p_{\epsilon} \leq  \epsilon$, and for  $\eta \leq \frac{1}{2L N}$, we have
\begin{align*}
\frac{1}{t+1} \sum_{k=0}^t \mathbb{E}[ \| \nabla F(\m{w}_k) \|^2 ]
&\leq \frac{2[F(\m{w}_0) - F^*]}{\eta (t + 1)} + 2 L\eta (1 + M_{\epsilon}) \epsilon.
\end{align*}
Hence, if $t \geq T$ for $T = \frac{[F(\m{w}_0) - F^*]}{(L\eta^2)(1 + M_{\epsilon})\epsilon}$, we have
\begin{align*}
\frac{1}{t+1} \sum_{k=0}^t \mathbb{E}[ \| \nabla F(\m{w}_k) \|^2 ]
&\leq 4 L\eta (1 + M_{\epsilon}) \epsilon.
\end{align*}
\end{cor}

%
%

\section{Numerical Experiments}\label{sec_experiment}

The purpose of this section is to numerically validate our
assumptions on $p_\epsilon$  as defined in Definition \ref{defn_p}.
We wish to show that there exists a small $\epsilon$ satisfying
\begin{equation}\label{eq:epsval}
1- p_\epsilon \approx \epsilon.\end{equation}
For our numerical
experiments, we consider the finite sum minimization problem
\begin{equation}\label{eq:finite_sum}
\min_{\m{w} \in \mathbb{R}^d} \left\{ F(\m{w}) = \frac{1}{n}
\sum_{i=1}^n f_i(\m{w}) \right\}.
\end{equation}

\begin{defn}
\label{defn_p_fin} Let $\m{w}_*$ be a stationary point of the objective
function $F(\m{w})$. For any given threshold $\epsilon > 0$, define
the set $\C{S}_{\epsilon}$ and its complement $\C{B}_{\epsilon}$
\begin{align*}
\C{S}_{\epsilon} &:= \left\{  i : \norm{ \nabla f_{i} (\m{w}_*) }^2
\leq \epsilon \right\}  \quad\quad \text{and} \quad\quad
\C{B}_{\epsilon} := [n] \backslash \C{S}_{\epsilon}.
\end{align*}
We also define the quantity $p_{\epsilon} :=
\frac{|\C{S}_{\epsilon}|}{n}$ that measures the size of the set
$\C{S}_{\epsilon}$ and the upper bound $M_{\epsilon}$
$$
\frac{1}{|\C{B}_{\epsilon}|} \sum_{i \in \C{B}_{\epsilon}}\| \nabla
f_{i} (\m{w}_*) \|^2 \leq M_{\epsilon}.
$$
\end{defn}

%

\subsection{Logistic Regression for Convex Case}

We consider $\ell_2$-regularized logistic regression problems with
\begin{align*}
f_i(\m{w}) = \log(1 + \exp(-y_i \m{x}_i^\top \m{w})) +
\frac{\lambda}{2} \| \m{w} \|^2,
\end{align*}
where the penalty parameter $\lambda$ is set to $1/n$, a widely-used
value in the literature \cite{nguyen2017sarah}. We conducted
experiments on popular datasets  \texttt{covtype, ijcnn1, w8a, a9a,
mushrooms, phishing, skin\_nonskin} from the LIBSVM website
\footnote{http://www.csie.ntu.edu.tw/$\sim$cjlin/libsvmtools/datasets/}
and \texttt{ijcnn2}
\footnote{http://mlbench.org/repository/data/viewslug/ijcnn1/}. The
optimal solution $\m{w}_*$ of the convex problem
(\ref{eq:finite_sum}) is found by using the full-batch
\texttt{L-BFGS}  method \cite{ln89} with the stopping criterion
$\|\nabla F(\m{w}_*) \|_2 \le 10^{-12}$. We then ran Algorithm
\ref{sgd_algorithm} using the  learning rate $\eta = 10^{-1}$ and
the batch-size $b=1$ and 100 epochs. The final solution given by the
SGD algorithm is denoted by $\m{w}_{SGD}$. We report the value of
$p_\epsilon$ defined in Definition \ref{defn_p_fin} expressed in
percentage form for different values of $\epsilon$.

As we can see from Table \ref{tab:pvalues_LR} that $\epsilon =
10^{-3}$ satisfies (\ref{eq:epsval}) for all cases. For datasets
\texttt{covtype, ijcnn1, ijcnn2,phishing} and
\texttt{skin\_nonskin}, $\epsilon$ can take a smaller value
$10^{-4}$. The small value for $\epsilon$ indicates that SGD with a
fixed step size can converge to a small neighborhood of the optimal
solution of $F$. The success of using SGD is illustrated, optimality
gaps $F(\m{w}_{SGD}) -
F(\m{w}_*)$ are small in our experiments.

\begin{table}[H]
\scriptsize \centering \caption{Percentage of $f_i$ with small
gradient value for different threshold $\epsilon$ (Logistic Regression)}
\label{tab:pvalues_LR}
\begin{tabular}{|c|c|c|c|c|c|c|c|}
\hline
Datasets & $F(\m{w}_{SGD}) - F(\m{w}_*)$ & $\epsilon = 10^{-2}$ & $\epsilon = 10^{-3}$ & $\epsilon = 10^{-4}$ & $\epsilon = 10^{-5}$ & $\epsilon = 10^{-6}$ & Train accuracy \\
\hline
\textbf{covtype} & $5\cdot 10^{-4}$  & 100\% & 100\% & 100\% & 99.9995\% &  54.9340\% & 0.7562 \\
\hline
\textbf{ijcnn1} & $1\cdot 10^{-4}$  & 100\% & 100\% & 100\% & 96.8201\% &  89.0197\% & 0.9219 \\
\hline
\textbf{ijcnn2} &  $2\cdot 10^{-4}$  & 100\% & 100\% & 100\% & 99.2874\% &  90.4565\% & 0.9228  \\
\hline
\textbf{w8a} &  $8\cdot 10^{-5}$  & 100\% & 99.9899\% & 99.4231\% & 98.3557\% &  92.7818\% & 0.9839 \\
 \hline
 \textbf{a9a} &  $4\cdot 10^{-3}$  & 100\% & 100\% & 84.0945\% & 58.5824\% &  40.0909\% & 0.8491 \\
 \hline
\textbf{mushrooms} &  $3\cdot 10^{-5}$  & 100\% & 100\% & 99.9261\% & 98.7568\% &  94.4239\% & 1.0000 \\
 \hline
 \textbf{phishing} &  $2\cdot 10^{-4}$ & 100\% & 100\% & 100\% & 89.9231\% &  73.8128\% & 0.9389 \\
 \hline
\textbf{skin\_nonskin} &  $4\cdot 10^{-5}$ & 100\% & 100\% & 100\% & 99.6331\% &  91.3730\% & 0.9076 \\
 \hline
\end{tabular}
\end{table}


We compare convergence rates of SGD (learning rate $\eta = 0.1 <
\frac{1}{2L}$) with \texttt{SVRG} \cite{SVRG} and \texttt{L-BFGS}
\cite{ln89} as shown in Figure \ref{fig_convg2}. We can observe that
SGD has better performance than SVRG and L-BFGS in the beginning
until it achieves $\Ocal(\epsilon)$ accuracy, for the value of
$\epsilon$ consistent to what is indicated in Table
\ref{tab:pvalues_LR}. We note that the values of $M_{\epsilon}$ for
all datasets should not exceed $10^{-2}$ according to Table
\ref{tab:pvalues_LR}.

 \begin{figure}[h]
 \centering
   \includegraphics[width=0.49500\textwidth]{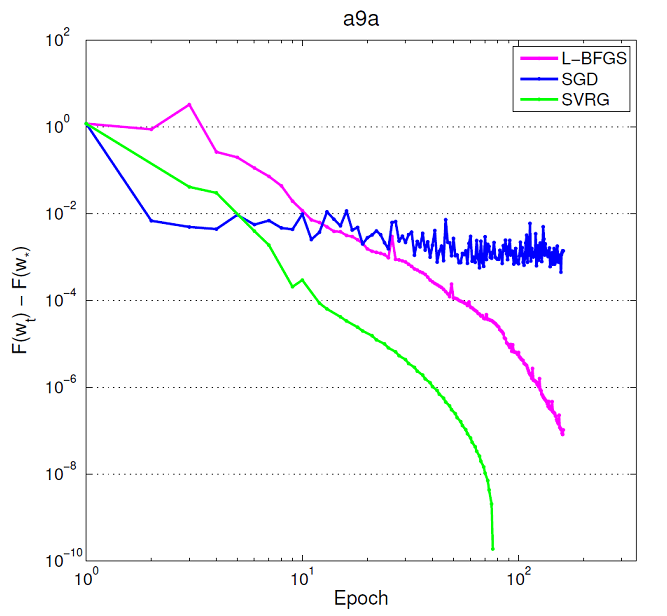}
   \includegraphics[width=0.49500\textwidth]{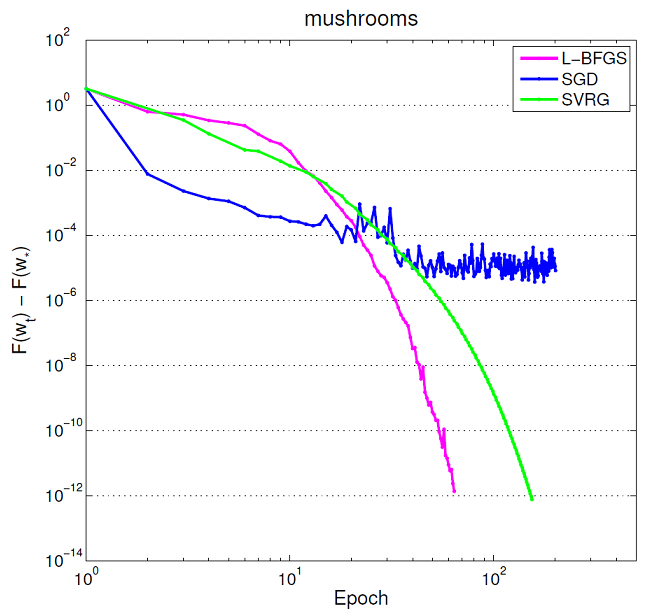}
   \includegraphics[width=0.49500\textwidth]{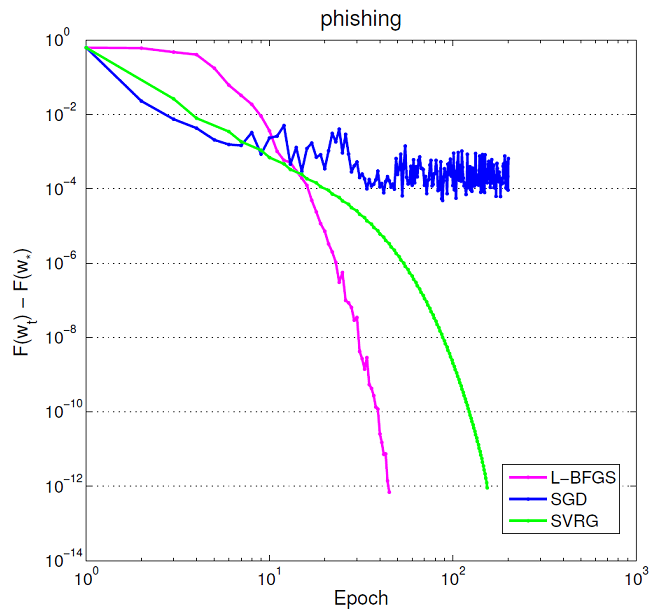}
   \includegraphics[width=0.49500\textwidth]{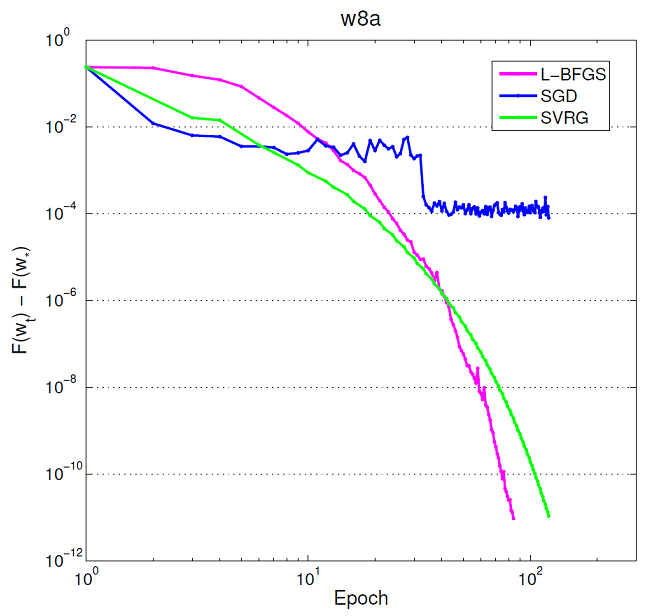}
   \caption{The convergence comparisons of SGD, SVRG, and L-BFGS}
  \label{fig_convg2}
 \end{figure}


%
%

\subsection{Neural Networks for Nonconvex Case}


For experiments with nonconvex problems we train DNNs using
 two standard network architectures: feed forward
network (FFN) and convolutional neural network (CNN). Configuration
of FNN includes $2$ dense layers each containing $256$ neurons
followed by a ReLU activation. The output layer consists of $c$
neurons with the softmax activation where $c$ is the number of classes. For CNN, we configure the
network to have $2$ convolutional layers followed by $2$ dense
layers. Convolutional layers contain a convolutional operator
followed by a ReLU activation and then a max pooling. The number of
filters of both the convolutional operators are set to $64$ and the
associated filter sizes are $5 \times 5$. Number of neurons in dense
layers are $384$ and $192$, respectively, and the activation used in
these layers is again ReLU. Throughout the simulations, we use popular datasets which include \texttt{MNIST} \footnote{http://yann.lecun.com/exdb/mnist/} (60000 training data images of size $28 \times 28$ contained in 10 classes), \texttt{SVHN} \footnote{http://ufldl.stanford.edu/housenumbers/} (73257 training images of size $32 \times 32$  contained in 10 classes), \texttt{CIFAR10} (50000 training color images of size $32 \times 32$ contained in 10 classes), and \texttt{CIFAR100} \footnote{https://www.cs.toronto.edu/~kriz/cifar.html} (50000 training color images of size $32 \times 32$ contained in 100 classes).


We trained the networks by the popular Adam algorithm with a
minibatch of size $32$ and reported the values of $p_{\epsilon}$ at
the last iteration $\m{w}_{Adam}$. In all our experiments, we did
not apply batch normalization and dropout techniques during the
training. Since the problem of interest is nonconvex, multiple local
minima could exist. We experimented with 10 seeds and reported the minimum result (minimum of the percentage of component functions with small gradient value). Table \ref{tab:pvalues} shows the values of
$p_\epsilon$ in terms of percentage for different thresholds
$\epsilon$. As is clear from the table, $p_\epsilon$ is close to $1$
for a sufficiently small $\epsilon$. It confirms that the majority
of component functions $f_i$ has negligible gradients at the final
solution of $F$.


\begin{table}[H]
\scriptsize \centering \caption{Percentage of $f_i$ with small
gradient value for different threshold $\epsilon$ (Neural Networks)}
\label{tab:pvalues}
\begin{tabular}{|c|c|l|c|c|c|c|c|c|}
\hline
Datasets & Architecture & $ \| \nabla F(\m{w}_{Adam}) \|^2$ &$\epsilon = 10^{-3}$ & $\epsilon = 10^{-5}$ &
$\epsilon = 10^{-7}$ & Train accuracy & $N$ & $M$\\
\hline
\textbf{MNIST} & \textbf{FFN}  & $1.3\cdot 10^{-15}$ & 100\% & 100\% & 99.99\% & 1.0000 & 6500 & $2.1\cdot 10^{-8}$ \\
\hline
\textbf{SVHN} &  \textbf{FFN} & $3.5\cdot 10^{-3}$ & 99.94\% & 99.92\% & 99.91\% & 0.9997 & 12000 & 500 \\
\hline
\textbf{MNIST} &  \textbf{CNN} & $1.6\cdot 10^{-17}$ & 100\% & 100\% & 100\% & 1.0000 & 6083 & $6.4\cdot 10^{-8}$ \\
 \hline
 \textbf{SVHN} &  \textbf{CNN} & $8.1\cdot 10^{-7}$ & 99.99\% & 99.98\% & 99.96\% & 0.9999 & 8068 & 0.18 \\
 \hline
\textbf{CIFAR10} &  \textbf{CNN} & $5.1\cdot 10^{-20}$ & 100\% & 100\% & 100\% & 1.0000 & 1205 & $8.7\cdot 10^{-14}$ \\
 \hline
 \textbf{CIFAR100} &  \textbf{CNN} & $5.5\cdot 10^{-2}$ & 99.50\% & 99.45\% & 99.42\% & 0.9988 & 984 & 3000 \\
 \hline
\end{tabular}
\end{table}

The value of $N$ is the estimation of $N$ in
\eqref{eq_ass_nonconvex_01}, which is shown in Section \ref{assump_verify}. We note that for some datasets and network structures,
Adam did not converge to a real local solution (SVHN-FFN and
CIFAR100-CNN) and Table \ref{tab:pvalues} shows only an
approximation of the behavior at the local solution.

\subsection{Nonconvex assumption verification}\label{assump_verify}

This section shows how to
estimate $N$. We are proving some numerical experiments to verify Assumption \ref{ass_nonconvex_00}. Let us define
\begin{align*}
r_t = \frac{\frac{1}{t+1} \sum_{k=0}^t\left( \frac{1}{n} \sum_{i=1}^n \| \nabla f_i(\m{w}_k) - \nabla f_i(\m{w}_*) \|^2 \right)}{\frac{1}{t+1}\sum_{k=0}^t \| F(\m{w}_k) \|^2}
\end{align*}

We show two plots to see behaviors of $r_t$ for MNIST (FFN) and
CIFAR10 (CNN) (others are reported in Table~\ref{tab:pvalues}. We
can observe from Figure \ref{fig_N} that $r_t$ is bounded above by a constant. (Note that $r_t \leq N$.) 

  \begin{figure}[h]
 \centering
 \includegraphics[width=0.49\textwidth]{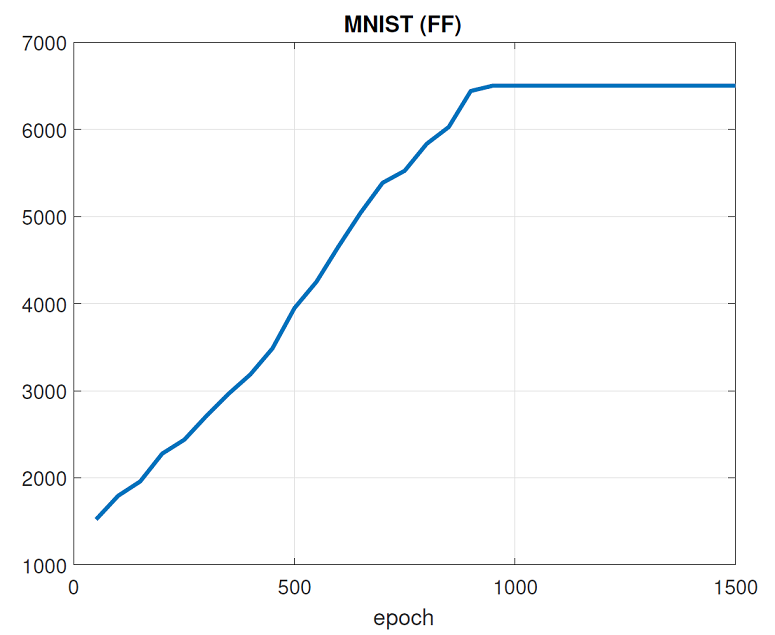}
 \includegraphics[width=0.49\textwidth]{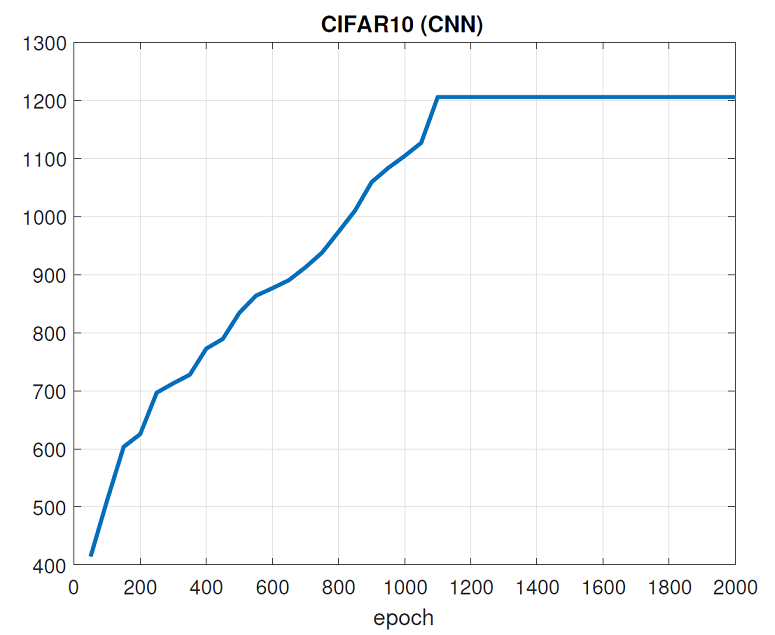}
   \caption{The behaviors of $r_t$}
  \label{fig_N}
 \end{figure}

\section{Conclusions}\label{sec_conclusion}
We have demonstrated that based on the behavior of the stochastic gradient estimates at or near the stationary points, SGD with fixed step size converges with the same rate as full gradient descent of the variance reduction methods, until it reaches the accuracy where the variance in the stochastic gradient estimates starts to dominate and prevents further convergence. In particular out assumption is that $1-\epsilon$ fraction of the stochastic gradient estimates have squared norm below $\epsilon$ at the solution. Note $\epsilon$ can be made arbitrarily small by increasing the  minibatch size $b$.
 Indeed we have the following lemma

\begin{lem}\label{lem:minibatch}
Let $\xi_1,\dots,\xi_b$ be i.i.d. with $\mathbb{E}[\nabla f(\m{w};\xi_i )] = \nabla F(\m{w})$, $i = 1,\dots,b$, for all $\m{w} \in \mathbb{R}^d$. Then,
\begin{align*}
\mathbb{E} \left[ \left\| \frac{1}{b} \sum_{i=1}^b \nabla f(\m{w}; \xi_i)  - \nabla F(\m{w}) \right\|^2  \right] = \frac{\mathbb{E}[ \| \nabla f(\m{w}; \xi_1) \|^2 ] - \| \nabla F(\m{w}) \|^2 }{b}. \tagthis \label{eq_useful_lemma}
\end{align*}
\end{lem}
It is easy to see that by choosing large $b$ the relation $1-p_\epsilon\leq \epsilon$ can be achieved for smaller values of $\epsilon$. In the limit for arbitrarily small $\epsilon$  we recover full gradient method and its convergence behavior.

\bibliography{reference}
\bibliographystyle{abbrv}


\clearpage
\section*{\textbf{Appendix}}


\section*{Useful Lemmas}

Let $\{\xi_i\}_{i=1}^b$ be i.i.d. random variables with $ \mathbb{E} [ f(\m{w};\xi_i) ] = F(\m{w})$.
From Definition \ref{defn_p}, we have
\begin{align*}
\mathbb{E} \left[ \| \m{g}_* \|^2 \right] &= \mathbb{E} \left[ \| \m{g}_* \|^2 \ | \ \| \m{g}_* \|^2 \leq \epsilon  \right] \cdot \mathbb{P} \left\{ \| \m{g}_* \|^2 \leq \epsilon \right\} + \mathbb{E} \left[ \| \m{g}_* \|^2 \ | \ \| \m{g}_* \|^2 > \epsilon  \right] \cdot \mathbb{P} \left\{ \m{g}_* \|^2 > \epsilon \right\} \\
& \leq p_{\epsilon} \epsilon + (1-p_{\epsilon}) M_{\epsilon}, \tagthis \label{eq_main_ass}
\end{align*}
where $\m{g}_* = \frac{1}{b} \sum_{i=1}^b \nabla f (\m{w}_* ; \xi_i)$.

\begin{lem}[\cite{nesterov2004}]
Suppose that $\phi$ is $L$-smooth. Then,
\begin{gather*}
\phi(\m{w}) \leq \phi(\m{w}') + \nabla \phi(\m{w}')^\top(\m{w} - \m{w}') +
\frac{L}{2}\|\m{w} - \m{w}'\|^2, \ \forall \m{w}, \m{w}' \in
\mathbb{R}^d.
\tagthis\label{eq:Lsmooth}
\end{gather*}
\end{lem}

\begin{lem}[\cite{nesterov2004}]
Suppose that $\phi$ is $L$-smooth and convex. Then,
\begin{gather*}
(\nabla \phi(\m{w}) - \nabla \phi(\m{w}'))^\top (\m{w} - \m{w}') \geq \frac{1}{L} \| \nabla \phi(\m{w}) - \nabla \phi(\m{w}') \|^2, \ \forall \m{w}, \m{w}' \in
\mathbb{R}^d.
\tagthis\label{eq_Lsmooth_convex}
\end{gather*}
\end{lem}

\begin{lem}[\cite{nesterov2004}]
Suppose that $\phi$ is $L$-smooth and convex. Then,
\begin{gather*}
\| \nabla \phi(\m{w}) \|^2 \leq 2 L ( \phi(\m{w}) - \phi(\m{w}_*) ), \forall
\m{w} \in \mathbb{R}^d, \tagthis \label{basic_prop_01}
\end{gather*}
where $\m{w}_* = \arg \min_{\m{w}} \phi(\m{w})$.
\end{lem}


\begin{lem}[\cite{nesterov2004}]
Suppose that $\phi$ is $\mu$-strongly convex. Then,
\begin{align*}
2\mu [ \phi(\m{w}) - \phi(\m{w}_*) ] \leq \| \nabla \phi(\m{w}) \|^2 \ , \
\forall \m{w} \in \mathbb{R}^d, \tagthis\label{eq:stronglyconvex}
\end{align*}
where $\m{w}_* = \arg \min_{\m{w}} \phi(\m{w})$.
\end{lem}

\begin{lem}[\cite{SVRG}]
Suppose that $f(\m{w};\xi)$ is $L$-smooth and
convex for every realization of $\xi$. Then,
\begin{align*}
\mathbb{E} [\| \nabla f(\m{w}; \xi) - \nabla
f (\m{w}_*; \xi) \|^2] \leq 2 L [F(\m{w}) - F(\m{w}_*)], \ \forall
\m{w} \in \mathbb{R}^d, \tagthis \label{basic_prop_02}
\end{align*}
where $\xi$ is a random variable, and $\m{w}_* = \arg \min_{\m{w}} F(\m{w})$.
\end{lem}

\begin{proof}
Given any $\xi$, for all $\m{w} \in \mathbb{R}^d$, consider
\begin{align*}
h(\m{w};\xi) := f(\m{w};\xi) - f(\m{w}_*;\xi) - \nabla f(\m{w}_*; \xi)^\top(\m{w} - \m{w}_*).
\end{align*}

Since $h(\m{w};\xi)$ is convex by $\m{w}$ and $\nabla h(\m{w}_*;\xi) = 0$, we have $h(\m{w}_*;\xi) = \min_{\m{w}} h(\m{w};\xi)$. Hence,
\begin{align*}
0 = h(\m{w}_*;\xi) & \leq \min_{\eta} \left[ h(\m{w} - \eta \nabla h(\m{w};\xi) ;\xi) \right] \\
& \overset{\eqref{eq:Lsmooth}}{\leq} \min_{\eta} \left[ h(\m{w}; \xi) - \eta \| \nabla h(\m{w}; \xi) \|^2 + \frac{L\eta^2}{2} \| \nabla h(\m{w}; \xi) \|^2    \right] \\
& = h(\m{w}; \xi) - \frac{1}{2L} \| \nabla h(\m{w}; \xi) \|^2.
\end{align*}

Hence,
\begin{align*}
\| \nabla f(\m{w}; \xi) - \nabla
f (\m{w}_*; \xi) \|^2 \leq 2 L [ f(\m{w};\xi) - f(\m{w}_*;\xi) - \nabla f(\m{w}_*; \xi)^\top(\m{w} - \m{w}_*) ].
\end{align*}

Taking the expectation with respect to $\xi$, we have
\begin{align*}
\mathbb{E} [\| \nabla f(\m{w}; \xi) - \nabla f (\m{w}_*; \xi) \|^2] \leq 2 L [F(\m{w}) - F(\m{w}_*)]. \ \ \ \ \ \ \ \ \ \ \ \ \ \ \ \qedhere
\end{align*}

\end{proof}

\section*{Proof of Lemma \ref{lem:minibatch}}

\textbf{Lemma \ref{lem:minibatch}}. \textit{Let $\xi_1,\dots,\xi_b$ be i.i.d. with $\mathbb{E}[\nabla f(\m{w};\xi_i )] = \nabla F(\m{w})$, $i = 1,\dots,b$, for all $\m{w} \in \mathbb{R}^d$. Then,}
\begin{align*}
\mathbb{E} \left[ \left\| \frac{1}{b} \sum_{i=1}^b \nabla f(\m{w}; \xi_i)  - \nabla F(\m{w}) \right\|^2  \right] = \frac{\mathbb{E}[ \| \nabla f(\m{w}; \xi_1) \|^2 ] - \| \nabla F(\m{w}) \|^2 }{b}. \tagthis \label{eq_useful_lemma}
\end{align*}

\begin{proof}
We are going to use mathematical induction to prove the result. With $b = 1$, it is easy to see
\begin{align*}
\mathbb{E} \left[ \left\| \nabla f(\m{w}; \xi_1)  - \nabla F(\m{w}) \right\|^2  \right] &= \mathbb{E}[ \| \nabla f(\m{w}; \xi_1) \|^2 ] - 2 \| \nabla F(\m{w}) \|^2 + \| \nabla F(\m{w}) \|^2 \\ &= \mathbb{E}[ \| \nabla f(\m{w}; \xi_1) \|^2 ] - \| \nabla F(\m{w}) \|^2 .
\end{align*}

Let assume that it is true with $b = m - 1$, we are going to show it is also true with $b = m$. We have
\begin{align*}
& \mathbb{E} \left[ \left\| \frac{1}{m} \sum_{i=1}^m \nabla f(\m{w}; \xi_i)  - \nabla F(\m{w}) \right\|^2  \right] \\
&= \mathbb{E} \left[ \left\|  \frac{\sum_{i=1}^{m-1} \nabla f(\m{w}; \xi_i) - (m-1) \nabla F(\m{w}) + (\nabla f(\m{w}; \xi_m) - \nabla F(\m{w}) ) }{m} \right\|^2 \right] \\
&= \frac{1}{m^2} \left(  \mathbb{E}\left[ \left\|  \sum_{i=1}^{m-1} \nabla f(\m{w}; \xi_i) - (m-1) \nabla F(\m{w}) \right\|^2  \right] + \mathbb{E} \left[ \left\| \nabla f(\m{w}; \xi_m) - \nabla F(\m{w}) \right\|^2 \right] \right) \\
& \qquad + \frac{1}{m} \mathbb{E} \left[ 2 \left(  \sum_{i=1}^{m-1} \nabla f(\m{w}; \xi_i) - (m-1) \nabla F(\m{w}) \right)^\top \left( \nabla f(\m{w}; \xi_m) - \nabla F(\m{w}) \right) \right] \\
&= \frac{1}{m^2} \left(  \mathbb{E}\left[ \left\|  \sum_{i=1}^{m-1} \nabla f(\m{w}; \xi_i) - (m-1) \nabla F(\m{w}) \right\|^2  \right] + \mathbb{E} \left[ \left\| \nabla f(\m{w}; \xi_m) - \nabla F(\m{w}) \right\|^2 \right] \right) \\
&= \frac{1}{m^2} \left( (m-1) \mathbb{E} [ \| \nabla f(\m{w}; \xi_1) \|^2] - (m-1)\| \nabla F(\m{w}) \|^2 + \mathbb{E} [ \| \nabla f(\m{w}; \xi_m) \|^2] - \| \nabla F(\m{w}) \|^2 \right) \\
&= \frac{1}{m} \left( \mathbb{E} [ \| \nabla f(\m{w}; \xi_1) \|^2] - \| \nabla F(\m{w}) \|^2  \right).
\end{align*}
The third and the last equalities follow since $\xi_1,\dots,\xi_b$ be i.i.d. with $\mathbb{E}[\nabla f(\m{w};\xi_i )] = \nabla F(\m{w})$. Therefore, the desired result is achieved.
\end{proof}

\section*{Proof of Theorem \ref{thm_sgd_str_convex_03}}

\textbf{Theorem \ref{thm_sgd_str_convex_03}}. \textit{Suppose that $F(\m{w})$ is $\mu$-strongly convex and $f(\m{w};\xi)$ is $L$-smooth and convex for every realization of $\xi$. Consider
Algorithm~\ref{sgd_algorithm} with $\eta \leq \frac{1}{L}$. Then, for any $\epsilon>0$
\begin{align*}
\mathbb{E} [ \|\m{w}_{t} - \m{w}_* \|^2] \leq (1- \mu\eta(1-\eta L) )^t
 \| \m{w}_{0} - \m{w}_* \|^2  + \frac{2\eta }{ \mu(1-\eta L)} p_{\epsilon} \epsilon +
\frac{2\eta}{ \mu(1-\eta L)} (1-p_{\epsilon})M_{\epsilon},
\end{align*}
where $\m{w}_* = \arg \min_{\m{w}} F(\m{w})$, and $p_{\epsilon}$ and $M_{\epsilon}$ are defined in \eqref{eq_Prob_p} and \eqref{eq_upper_bound}, respectively.}
\begin{proof}
We have
\begin{align*}
& \|\m{w}_{t+1} - \m{w}_* \|^2 = \left \|\m{w}_t - \eta \m{g}_t - \m{w}_* \right \|^2 \\
&= \| \m{w}_{t} - \m{w}_* \|^2 - 2\eta \m{g}_t^\top (\m{w}_t - \m{w}_*) + \eta^2 \left \| \m{g}_t \right \|^2 \\
&= \| \m{w}_{t} - \m{w}_* \|^2 - 2\eta \frac{1}{b} \sum_{i=1}^b \nabla f(\m{w}_t; \xi_{t,i})^\top (\m{w}_t - \m{w}_*) + \eta^2 \left \| \frac{1}{b} \sum_{i=1}^b \nabla f(\m{w}_t; \xi_{t,i}) \right \|^2 \\
&\leq \| \m{w}_{t} - \m{w}_* \|^2 - 2\eta \frac{1}{b} \sum_{i=1}^b \nabla f(\m{w}_t; \xi_{t,i})^\top (\m{w}_t - \m{w}_*) + 2 \eta^2 \left \| \frac{1}{b} \sum_{i=1}^b \left( \nabla f(\m{w}_t; \xi_{t,i}) - \nabla f(\m{w}_*; \xi_{t,i}) \right) \right \|^2 \\ & \qquad + 2 \eta^2 \left \| \frac{1}{b} \sum_{i=1}^b \nabla f(\m{w}_*; \xi_{t,i}) \right \|^2 \\
&\leq \| \m{w}_{t} - \m{w}_* \|^2 - 2\eta \frac{1}{b} \sum_{i=1}^b \nabla f(\m{w}_t; \xi_{t,i})^\top (\m{w}_t - \m{w}_*) + 2 \eta^2 \frac{1}{b} \sum_{i=1}^b \left \|  \nabla f(\m{w}_t; \xi_{t,i}) - \nabla f(\m{w}_*; \xi_{t,i})  \right \|^2 \\ & \qquad + 2 \eta^2 \left \| \frac{1}{b} \sum_{i=1}^b \nabla f(\m{w}_*; \xi_{t,i}) \right \|^2 \tagthis \label{eq_in_thm1} \\
&\overset{\eqref{eq_Lsmooth_convex}}{\leq} \| \m{w}_{t} - \m{w}_* \|^2 - 2\eta \frac{1}{b} \sum_{i=1}^b \nabla f(\m{w}_t; \xi_{t,i})^\top (\m{w}_t - \m{w}_*) + 2 \eta^2 L \frac{1}{b} \sum_{i=1}^b (\nabla f(\m{w}_t; \xi_{t,i}) - \nabla f(\m{w}_*; \xi_{t,i}))^\top (\m{w}_t - \m{w}_*) \\ & \qquad + 2 \eta^2 \left \| \frac{1}{b} \sum_{i=1}^b \nabla f(\m{w}_*; \xi_{t,i}) \right \|^2.
\end{align*}

Hence, by taking the expectation, conditioned on $\mathcal{F}_t =
\sigma(\m{w}_0,\m{w}_1,\dots,\m{w}_t)$ (which is the
$\sigma$-algebra generated by $\m{w}_0,\m{w}_1,\dots,\m{w}_t$),  we have
\begin{align*}
\mathbb{E} [ \|\m{w}_{t+1} - \m{w}_* \|^2  | \mathcal{F}_t ] & \leq  \| \m{w}_{t} - \m{w}_* \|^2 - 2\eta ( 1 - \eta L) \nabla F(\m{w}_t)^\top (\m{w}_t - \m{w}_*) + 2 \eta^2 \mathbb{E} \left[ \left \| \frac{1}{b} \sum_{i=1}^b \nabla f(\m{w}_*; \xi_{t,i}) \right \|^2 \Big | \mathcal{F}_t \right] \\
& \overset{\eqref{eq:stronglyconvex_00}}{\leq}  (1- \mu\eta( 1 - \eta L)) \| \m{w}_{t} - \m{w}_* \|^2 - 2\eta(1 - \eta L) [F(\m{w}_t) - F(\m{w}_*)] \\ & \qquad \qquad + 2 \eta^2 \mathbb{E} \left[ \left \| \frac{1}{b} \sum_{i=1}^b \nabla f(\m{w}_*; \xi_{t,i}) \right \|^2 \right] \\
& \overset{\eta \leq 1/L,\eqref{eq_main_ass}}{\leq} (1- \mu\eta(1-\eta L)) \| \m{w}_{t} - \m{w}_* \|^2 + 2\eta^2 p_{\epsilon} \epsilon + 2\eta^2 (1-p_{\epsilon}) M_{\epsilon}.
\end{align*}
The first inequality follows since
\begin{align*}
\mathbb{E}\left[ \frac{1}{b} \sum_{i=1}^b \nabla f(\m{w}_t; \xi_{t,i}) \Big | \mathcal{F}_t \right] = \mathbb{E}\left[ \frac{1}{b} \sum_{i=1}^b \left( \nabla f (\m{w}_t; \xi_{t,i}) - \nabla f (\m{w}_*; \xi_{t,i}) \right) \Big | \mathcal{F}_t \right] = \nabla F(\m{w}_t).
\end{align*}

We note in the second equality that $\mathbb{E} \left[ \left \| \frac{1}{b} \sum_{i=1}^b \nabla f(\m{w}_*; \xi_{t,i}) \right \|^2 \Big | \mathcal{F}_t \right] = \mathbb{E} \left[ \left \| \frac{1}{b} \sum_{i=1}^b \nabla f(\m{w}_*; \xi_{t,i}) \right \|^2 \right]$ since $\xi_{t,i}$ is independent of $\mathcal{F}_t$. By taking the expectation for both sides of the above equation, we obtain
\begin{align*}
\mathbb{E} [ \|\m{w}_{t+1} - \m{w}_* \|^2] \leq (1- \mu\eta(1-\eta L))
\mathbb{E} [ \| \m{w}_{t} - \m{w}_* \|^2 ] + 2\eta^2 p_{\epsilon} \epsilon +
2\eta^2 (1-p_{\epsilon})M_{\epsilon}.
\end{align*}

Hence, we conclude
\begin{align*}
\mathbb{E} [ \|\m{w}_{t+1} - \m{w}_* \|^2] \leq (1- \mu\eta(1-\eta L) )^{t+1}
\| \m{w}_{0} - \m{w}_* \|^2  + \frac{2\eta }{\mu(1-\eta L)} p_{\epsilon} \epsilon +
\frac{2\eta}{\mu(1-\eta L)} (1-p_{\epsilon})M_{\epsilon}.
\end{align*}
\end{proof}
%
%
%

\section*{Proof of Theorem \ref{thm_sgd_convex_01}}

\textbf{Theorem \ref{thm_sgd_convex_01}}. \textit{Suppose that $f(\m{w};\xi)$ is $L$-smooth and convex for every realization of $\xi$.
Consider Algorithm~\ref{sgd_algorithm} with $\eta < \frac{1}{L}$. Then for any $\epsilon>0$, we have
\begin{align*}
\frac{1}{t+1} \sum_{k=0}^t \mathbb{E} [F(\m{w}_k) - F(\m{w}_*)] &
\leq \frac{\| \m{w}_0 - \m{w}_* \|^2}{2\eta(1 - \eta L) t}  +
\frac{\eta}{(1 - \eta L)} p_{\epsilon} \epsilon + \frac{\eta M_{\epsilon}}{(1 - \eta L)} (1 - p_{\epsilon}), 
\end{align*}
where $\m{w}_* $ is any optimal solution of $F(\m{w})$, and  $p_{\epsilon}$ and $M_{\epsilon}$ are defined in \eqref{eq_Prob_p} and \eqref{eq_upper_bound}, respectively.}

\begin{proof}
If $\phi$ is convex, then
\begin{align*}
\phi(\m{w}) - \phi(\m{w}') \geq \nabla \phi(\m{w}')^\top (\m{w} - \m{w}'), \
\forall \m{w}, \m{w}' \in \mathbb{R}^d. \tagthis\label{eq:convex_00}
\end{align*}

From the proof of Theorem \ref{thm_sgd_str_convex_03}, we could have
\begin{align*}
\mathbb{E} [ \|\m{w}_{t+1} - \m{w}_* \|^2  | \mathcal{F}_t ] & \leq  \| \m{w}_{t} - \m{w}_* \|^2 - 2\eta ( 1 - \eta L) \nabla F(\m{w}_t)^\top (\m{w}_t - \m{w}_*) \\ & \qquad + 2 \eta^2 \mathbb{E} \left[ \left \| \frac{1}{b} \sum_{i=1}^b \nabla f(\m{w}_*; \xi_{t,i}) \right \|^2 \Big | \mathcal{F}_t \right] \\
& \overset{\eqref{eq:convex_00},\eqref{eq_main_ass}}{\leq} \| \m{w}_{t} - \m{w}_* \|^2 - 2\eta(1 - \eta L) [F(\m{w}_t) - F(\m{w}_*)] + 2\eta^2 p_{\epsilon}\epsilon + 2\eta^2 (1-p_{\epsilon}) M_{\epsilon}.
\end{align*}

Taking the expectation for both sides of the above equation yields
\begin{align*}
\mathbb{E} [ \|\m{w}_{t+1} - \m{w}_* \|^2] & \leq \mathbb{E} [\|
\m{w}_{t} - \m{w}_* \|^2] - 2\eta(1 - \eta L) \mathbb{E}
[F(\m{w}_t) - F(\m{w}_*)] + 2\eta^2 p_{\epsilon}\epsilon + 2\eta^2
(1-p_{\epsilon}) M_{\epsilon}.
\end{align*}

With $\eta < \frac{1}{L}$, one obtains
\begin{align*}
\mathbb{E} [F(\m{w}_t) - F(\m{w}_*)] & \leq \frac{1}{2\eta(1 - \eta
L)} \Big( \mathbb{E} [\| \m{w}_{t} - \m{w}_* \|^2] - \mathbb{E} [
\|\m{w}_{t+1} - \m{w}_* \|^2] \Big) \\ & \qquad + \frac{\eta}{(1 -
\eta L)} p_{\epsilon} \epsilon + \frac{\eta M_{\epsilon}}{(1 - \eta L)} (1 - p_{\epsilon}).
\end{align*}

By summing from $k = 0,\dots,t$ and averaging, we have
\begin{align*}
\frac{1}{t+1} \sum_{k=0}^t \mathbb{E} [F(\m{w}_k) - F(\m{w}_*)] &
\leq \frac{1}{2\eta(1 - \eta L)(t+1)} \| \m{w}_0 - \m{w}_* \|^2  +
\frac{\eta}{(1 - \eta L)} p_{\epsilon} \epsilon + \frac{\eta M_{\epsilon}}{(1 - \eta L)} (1 - p_{\epsilon}).
\end{align*}


\end{proof}

\section*{Proof of Theorem \ref{thm_nonconvex_05}}

\textbf{Theorem \ref{thm_nonconvex_05}}. \textit{Let Assumption \ref{ass_nonconvex_00} hold for some $N>0$. Suppose that $F$ is $L$-smooth. Consider Algorithm \ref{sgd_algorithm} with $\eta < \frac{1}{LN}$. Then, for any $\epsilon > 0$, we have
\begin{align*}
\frac{1}{t+1} \sum_{k=0}^t \mathbb{E}[ \| \nabla F(\m{w}_k) \|^2 ]
&\leq \frac{[F(\m{w}_0) - F^*]}{\eta\left(1 - L\eta N \right) (t + 1)} + \frac{L\eta}{ \left(1 - L\eta N \right)} \epsilon + \frac{L\eta M_{\epsilon}}{\left(1 - L\eta N \right)} ( 1 - p_{\epsilon}),
\end{align*}
where $F^*$ is any lower bound of $F$; and $p_{\epsilon}$  and $M_{\epsilon}$ are defined in \eqref{eq_Prob_p2} and \eqref{eq_upper_bound2} respectively.}

\begin{proof}
Let us assume that, there exists a local minima $\m{w}_*$ of $F(\m{w})$. We have
\begin{align*}
\mathbb{E}[F(\m{w}_{t+1}) | \mathcal{F}_t ] & = \mathbb{E}[F(\m{w}_t
- \eta \m{g}_t) | \mathcal{F}_t ]
 \overset{\eqref{eq:Lsmooth}}{\leq} F(\m{w}_t) - \eta \| \nabla F(\m{w}_t) \|^2 + \frac{L\eta^2}{2} \mathbb{E} \left[ \left \| \frac{1}{b} \sum_{i=1}^b \nabla f (\m{w}_t; \xi_{t,i}) \right \|^2 \Big | \mathcal{F}_t \right] \\
& \leq F(\m{w}_t) - \eta \| \nabla F(\m{w}_t) \|^2 + L\eta^2 \mathbb{E} \left[ \left \| \frac{1}{b} \sum_{i=1}^b ( \nabla f (\m{w}_t; \xi_{t,i}) - \nabla f(\m{w}_*; \xi_{t,i}) ) \right \|^2 \Big | \mathcal{F}_t \right] \\ & \qquad + L\eta^2 \mathbb{E} \left[ \left \| \frac{1}{b} \sum_{i=1}^b \nabla f (\m{w}_*; \xi_{t,i}) \right \|^2 \Big | \mathcal{F}_t \right] \\
& \leq F(\m{w}_t) - \eta \| \nabla F(\m{w}_t) \|^2 + L\eta^2 \mathbb{E} \left[ \left \| \frac{1}{b} \sum_{i=1}^b ( \nabla f (\m{w}_t; \xi_{t,i}) - \nabla f(\m{w}_*; \xi_{t,i}) ) \right \|^2 \Big | \mathcal{F}_t \right] \\ & \qquad + L\eta^2 \epsilon + L\eta^2 (1 - p_{\epsilon}) M_{\epsilon}.
\end{align*}

%
%

By summing from $k = 0,\dots,t$ and averaging, we have
\begin{align*}
\frac{1}{t+1} \sum_{k=0}^t \mathbb{E}[F(\m{w}_{k+1}) | \mathcal{F}_k ] & \leq \frac{1}{t+1} \sum_{k=0}^t F(\m{w}_k) - \eta \frac{1}{t+1} \sum_{k=0}^t \| \nabla F(\m{w}_k) \|^2 \\ & \qquad + L\eta^2 \frac{1}{t+1} \sum_{k=0}^t \left( \mathbb{E} \left[ \left \| \frac{1}{b} \sum_{i=1}^b ( \nabla f (\m{w}_k; \xi_{k,i}) - \nabla f(\m{w}_*; \xi_{k,i}) ) \right \|^2 \Big | \mathcal{F}_k \right] \right) \\ & \qquad + L\eta^2 \epsilon + L\eta^2 (1 - p_{\epsilon}) M_{\epsilon} \\
& \overset{\eqref{eq_ass_nonconvex_01}}{\leq} \frac{1}{t+1} \sum_{k=0}^t F(\m{w}_k) - \eta\left(1 -
L\eta N \right) \frac{1}{t+1} \sum_{k=0}^t \| \nabla F(\m{w}_k) \|^2 \\ & \qquad + L\eta^2 \epsilon + L\eta^2 (1 - p_{\epsilon}) M_{\epsilon}.
\end{align*}

Taking the expectation for the above equation, we have
\begin{align*}
\frac{1}{t+1} \sum_{k=0}^t \mathbb{E}[F(\m{w}_{k+1})] & \leq \frac{1}{t+1} \sum_{k=0}^t \mathbb{E}[F(\m{w}_{k})] - \eta\left(1 -
L\eta N \right) \frac{1}{t+1} \sum_{k=0}^t \mathbb{E}[ \| \nabla F(\m{w}_k) \|^2 ] \\ & \qquad + L\eta^2 \epsilon + L\eta^2 (1 - p_{\epsilon}) M_{\epsilon}.
\end{align*}

Hence, with $\eta < \frac{1}{L N}$, we have
\begin{align*}
\frac{1}{t+1} \sum_{k=0}^t \mathbb{E}[ \| \nabla F(\m{w}_k) \|^2 ]
& \leq \frac{\left[\mathbb{E}[F(\m{w}_{0})] - \mathbb{E}[F(\m{w}_{t+1})]\right]}{\eta\left(1 - L\eta N \right) (t + 1)}  + \frac{L\eta}{ \left(1 - L\eta N \right)} \epsilon + \frac{L\eta M_{\epsilon}}{ \left(1 - L\eta N \right)} ( 1 - p_{\epsilon}) \\
& \leq \frac{[F(\m{w}_0) - F^*]}{\eta\left(1 - L\eta N \right) (t + 1)}  + \frac{L\eta}{ \left(1 - L\eta N \right)} \epsilon + \frac{L\eta M_{\epsilon}}{\left(1 - L\eta N \right)} ( 1 - p_{\epsilon}),
\end{align*}
where $F^*$ is any lower bound of $F$.
\end{proof}

\section*{Proof of Corollary \ref{cor_sgd_str_convex_03_II}}

\textbf{Corollary \ref{cor_sgd_str_convex_03_II}}. \textit{For any $\epsilon$ such that $1 - p_{\epsilon} \leq  \epsilon$, and for Algorithm~\ref{sgd_algorithm} with $\eta \leq \frac{1}{2L}$, we have}
\begin{align*}
\mathbb{E} [ \|\m{w}_t - \m{w}_* \|^2 ] & \leq (1- \mu\eta)^t \|
\m{w}_{0} - \m{w}_* \|^2 + \frac{2\eta}{\mu}\left (1+ M_{\epsilon}\right )\epsilon.
\end{align*}
\textit{Furthermore if $t\geq T$ for $T =  \frac{1}{\mu \eta} \log \left( \frac{\mu\| \m{w}_{0} - \m{w}_* \|^2}{2\eta\left (1+ M_{\epsilon}\right )\epsilon} \right)$, then}
 \begin{align*}
 \mathbb{E} [
\|\m{w}_{t} - \m{w}_* \|^2 ] \leq \frac{4\eta}{ \mu}\left (1+M_{\epsilon}\right )  \epsilon .
\end{align*}

\begin{proof}
Taking the expectation, conditioning on $\mathcal{F}_t =
\sigma(\m{w}_0,\m{w}_1,\dots,\m{w}_t)$ to \eqref{eq_in_thm1}, we have
\begin{align*}
\mathbb{E} [ \|\m{w}_{t+1} - \m{w}_* \|^2  | \mathcal{F}_t ] & \leq  \| \m{w}_{t} - \m{w}_* \|^2 - 2\eta \nabla F(\m{w}_t)^\top (\m{w}_t - \m{w}_*) \\ & \qquad + 2 \eta^2 \mathbb{E}[ \| \nabla f(\m{w}_t; \xi_{t,1}) - \nabla f(\m{w}_*; \xi_{t,1}) \|] + 2 \eta^2 \mathbb{E} \left[ \left \| \frac{1}{b} \sum_{i=1}^b \nabla f(\m{w}_*; \xi_{t,i}) \right \|^2 \Big | \mathcal{F}_t \right] \\
& \overset{\eqref{eq:stronglyconvex_00},\eqref{basic_prop_02}}{\leq} (1 - \mu \eta)  \| \m{w}_{t} - \m{w}_* \|^2 - 2\eta (1 - 2\eta L) [ F(\m{w}_t) - F(\m{w}_*) ] \\ & \qquad + 2 \eta^2 \mathbb{E} \left[ \left \| \frac{1}{b} \sum_{i=1}^b \nabla f(\m{w}_*; \xi_{t,i}) \right \|^2 \right] \\
& \overset{\eta \leq \frac{1}{2L},\eqref{eq_main_ass}}{\leq} (1 - \mu \eta)  \| \m{w}_{t} - \m{w}_* \|^2 + 2\eta^2 p_{\epsilon} \epsilon + 2\eta^2 (1 - p_{\epsilon})M_{\epsilon}.
\end{align*}
The first inequality follows since $\{\xi_{i,i}\}_{i=1}^b$ are i.i.d. random variables. Hence, we have
\begin{align*}
\mathbb{E} [ \|\m{w}_{t+1} - \m{w}_* \|^2] \leq (1- \mu\eta )^{t+1}
\| \m{w}_{0} - \m{w}_* \|^2  + \frac{2\eta }{\mu} p_{\epsilon} \epsilon +
\frac{2\eta}{\mu} (1-p_{\epsilon})M_{\epsilon}.
\end{align*}

Therefore,
\begin{align*}
\mathbb{E} [ \|\m{w}_{t} - \m{w}_* \|^2] & \leq (1- \mu\eta )^t
 \| \m{w}_{0} - \m{w}_* \|^2  + \frac{2\eta }{ \mu} p_{\epsilon} \epsilon +
\frac{2\eta}{ \mu} (1-p_{\epsilon})M_{\epsilon} \\
& \leq (1- \mu\eta )^t
 \| \m{w}_{0} - \m{w}_* \|^2 + \frac{2\eta }{ \mu}(1 + M_{\epsilon})\epsilon,
\end{align*}
where the last inequality follows since $1-p_{\epsilon} \leq \epsilon$.

First, we would like to find a $T$ such that
\begin{align*}
(1- \mu\eta)^T \| \m{w}_{0} - \m{w}_* \|^2 = \frac{2\eta}{\mu}\left (1+ M_{\epsilon}\right )\epsilon.
\end{align*}

Taking $\log$ for both sides, we have
\begin{align*}
T \log(1 - \mu \eta) + \log\left( \| \m{w}_{0} - \m{w}_* \|^2 \right) = \log \left( \frac{2\eta}{\mu}\left (1+ M_{\epsilon} \right )\epsilon \right).
\end{align*}

Hence,
\begin{align*}
T & = - \frac{1}{\log(1 - \mu \eta)} \log \left( \frac{\mu\| \m{w}_{0} - \m{w}_* \|^2}{2\eta\left (1+ M_{\epsilon}\right )\epsilon} \right) \leq \frac{1}{\mu \eta} \log \left( \frac{\mu\| \m{w}_{0} - \m{w}_* \|^2}{2\eta\left (1+ M_{\epsilon}\right )\epsilon} \right),
\end{align*}
where the last inequality follows since $-1/\log(1-x) \leq 1/x$ for $0 < x \leq 1$. Hence, if $t \geq T$ for $T =  \frac{1}{\mu \eta} \log \left( \frac{\mu\| \m{w}_{0} - \m{w}_* \|^2}{2\eta\left (1+ M_{\epsilon}\right )\epsilon} \right)$, then
\begin{align*}
\mathbb{E} [ \|\m{w}_{t} - \m{w}_* \|^2 ] \leq \frac{2\eta}{\mu}\left (1+M_{\epsilon}\right )  \epsilon + \frac{2\eta}{ \mu}\left (1+M_{\epsilon}\right ) \epsilon = \frac{4\eta}{ \mu}\left (1+M_{\epsilon}\right ) \epsilon.
\end{align*}
\end{proof}
%
%

\section*{Proof of Corollary \ref{cor_sgd_convex_01}}

\textbf{Corollary \ref{cor_sgd_convex_01}}. \textit{If $f(\m{w};\xi)$ is $L$-smooth and convex for every realization of $\xi$, then for any $\epsilon$ such  that $1 - p_{\epsilon} \leq  \epsilon$,  and $\eta \leq
\frac{1}{2L}$, it holds that}
\begin{align*}
\frac{1}{t+1} \sum_{k=0}^t \mathbb{E} [F(\m{w}_k) - F(\m{w}_*)] &
\leq \frac{\| \m{w}_0 - \m{w}_* \|^2}{\eta t}  +
2\eta \left(1+ M_{\epsilon}  \right) \epsilon.
\end{align*}
\textit{Hence, if $t\geq T$ for  $T = \frac{ \| \m{w}_0 - \m{w}_* \|^2}{(2\eta^2) (1+M_{\epsilon})\epsilon}$, we have}
\begin{align*}
\frac{1}{t+1} \sum_{k=0}^t \mathbb{E} [F(\m{w}_k) - F(\m{w}_*)] \leq
4\eta \left(1+M_{\epsilon}\right) \epsilon.
\end{align*}

\begin{proof}
By Theorem \ref{thm_sgd_convex_01}, with $\eta \leq \frac{1}{2L}$, we have
\begin{align*}
\frac{1}{t+1} \sum_{k=0}^t \mathbb{E} [F(\m{w}_k) - F(\m{w}_*)] &
\leq \frac{\| \m{w}_0 - \m{w}_* \|^2}{2\eta(1 - \eta L) t}  +
\frac{\eta}{(1 - \eta L)} p_{\epsilon} \epsilon + \frac{\eta M_{\epsilon}}{(1 - \eta L)} (1 - p_{\epsilon}) \\
& \leq \frac{2 \| \m{w}_0 - \m{w}_* \|^2}{2\eta t}  +
2\eta p_{\epsilon} \epsilon + 2\eta M_{\epsilon} (1 - p_{\epsilon}) \\
& \leq \frac{ \| \m{w}_0 - \m{w}_* \|^2}{\eta t} + 2\eta(1 + M_{\epsilon})\epsilon.
\end{align*}

Similar to the proof of Corollary \ref{cor_sgd_str_convex_03_II}, we want to find a $T$ such that
\begin{align*}
\frac{\| \m{w}_0 - \m{w}_* \|^2}{\eta T}  = 2\eta(1 + M_{\epsilon})\epsilon.
\end{align*}

It is easy to see that if $t \geq T$ for $T = \frac{ \| \m{w}_0 - \m{w}_* \|^2}{(2\eta^2)(1 + M_{\epsilon})\epsilon}$, then
\begin{align*}
\frac{1}{t+1} \sum_{k=0}^t \mathbb{E} [F(\m{w}_k) - F(\m{w}_*)] \leq
2\eta(1 + M_{\epsilon})\epsilon + 2\eta(1 + M_{\epsilon})\epsilon = 4\eta(1 + M_{\epsilon})\epsilon.
\end{align*}
\end{proof}

\section*{Proof of Corollary \ref{cor_nonconvex_05}}

\textbf{Corollary \ref{cor_nonconvex_05}}. \textit{Let Assumption \ref{ass_nonconvex_00} hold and $p_{\epsilon}$ and $M_{\epsilon}$ be defined as in  \eqref{eq_Prob_p2} and \eqref{eq_upper_bound2}.  For any $\epsilon$ such that $1 - p_{\epsilon} \leq  \epsilon$, and for  $\eta \leq \frac{1}{2L N}$, we have}
\begin{align*}
\frac{1}{t+1} \sum_{k=0}^t \mathbb{E}[ \| \nabla F(\m{w}_k) \|^2 ]
&\leq \frac{2[F(\m{w}_0) - F^*]}{\eta (t + 1)} + 2 L\eta (1 + M_{\epsilon}) \epsilon.
\end{align*}
\textit{Hence, if $t \geq T$ for $T = \frac{[F(\m{w}_0) - F^*]}{(L\eta^2)(1 + M_{\epsilon})\epsilon}$, we have}
\begin{align*}
\frac{1}{t+1} \sum_{k=0}^t \mathbb{E}[ \| \nabla F(\m{w}_k) \|^2 ]
&\leq 4 L\eta (1 + M_{\epsilon}) \epsilon.
\end{align*}

\begin{proof}
By Theorem \ref{thm_nonconvex_05}, with $\eta \leq \frac{1}{2LN}$, we have
\begin{align*}
\frac{1}{t+1} \sum_{k=0}^t \mathbb{E}[ \| \nabla F(\m{w}_k) \|^2 ]
& \leq \frac{[F(\m{w}_0) - F^*]}{\eta\left(1 - L\eta N \right) (t + 1)}  + \frac{L\eta}{ \left(1 - L\eta N \right)} \epsilon + \frac{L\eta M_{\epsilon}}{\left(1 - L\eta N \right)} ( 1 - p_{\epsilon}) \\
& \leq \frac{2[F(\m{w}_0) - F^*]}{\eta (t + 1)}  + 2 L\eta \epsilon + 2 L\eta M_{\epsilon} ( 1 - p_{\epsilon}) \\
& \leq \frac{2[F(\m{w}_0) - F^*]}{\eta (t + 1)}  + 2 L\eta(1 + M_{\epsilon}) \epsilon.
\end{align*}

Similar to the proof of Corollaries \ref{cor_sgd_str_convex_03_II} and \ref{cor_sgd_convex_01}, we want to find a $T$ such that
\begin{align*}
\frac{2[F(\m{w}_0) - F^*]}{\eta T} = 2 L\eta (1 + M_{\epsilon}) \epsilon.
\end{align*}

It is easy to see that if $t \geq T$ for $T = \frac{[F(\m{w}_0) - F^*]}{(L\eta^2)(1 + M_{\epsilon})\epsilon}$, then
\begin{align*}
\frac{1}{t+1} \sum_{k=0}^t \mathbb{E}[ \| \nabla F(\m{w}_k) \|^2 ]
&\leq 2 L\eta (1 + M_{\epsilon}) \epsilon + 2 L\eta (1 + M_{\epsilon}) \epsilon = 4 L\eta (1 + M_{\epsilon}) \epsilon.
\end{align*}
\end{proof}
\end{document}